\documentclass{article}  

\newenvironment{tightList}
{\begin{list}{}{\partopsep=\baselineskip
	\parskip=0pt
	\parsep=0pt
	\topsep=0pt
	\itemsep=0pt
	\labelwidth=0pt
	\itemindent=-10pt}}
{\end{list}}

\usepackage{amsmath,amsthm,amsfonts,amssymb,verbatim, stmaryrd} 
\newcommand{\Real}{\mathbb{R}}

\usepackage{float}
\usepackage{graphicx}

\usepackage[margin=2in]{geometry}

\newtheorem{theorem}{Theorem}[section]
\newtheorem{definition}[theorem]{Definition}

\newtheorem{prop}[theorem]{Proposition}

\usepackage{url}

\usepackage{datetime}
\yyyymmdddate

\usepackage{amsmath}

\usepackage{tikz-cd}

\makeatletter
\renewcommand*{\@fnsymbol}[1]{\ensuremath{\ifcase#1\or \ \or *\or \dagger\or \ddagger\or
   \mathsection\or \mathparagraph\or \|\or **\or \dagger\dagger
   \or \ddagger\ddagger \else\@ctrerr\fi}}
\makeatother

\DeclareMathOperator{\ttv}{\operatorname{TTV}}

\newcommand{\ic}[1]{\mathcal{#1}}
\newcommand{\cC}{\ic{C}}
\newcommand{\cD}{\ic{D}}
\newcommand{\cS}{\ic{S}}

\newcommand{\kernel}{\operatorname{\mathcal{K}}}
\newcommand{\cokernel}{\operatorname{{\mathcal{C}o}\mathcal{K}}}

\usepackage{mathrsfs}   
\newcommand{\scale}{\mathscr{M}}

\newcommand{\per}[1]{\operatorname{per}(#1)}
\newcommand{\ap}[1]{\operatorname{\mathcal{A}}(#1)}
\newcommand{\prgn}[1]{\operatorname{\mathcal{P}}(#1)}

\newcommand{\perboom}[3]{#1\triangleright#2\triangleleft#3} 

\newcommand{\halfsupport}[1]{{#1}^\bullet}

\newcommand{\intervalOC}[2]{\boldsymbol{(}#1\ #2\boldsymbol{]}}

\title{Persistence Lenses: \\Segmentation, Simplification, Vectorization, \\Scale Space and Fractal Analysis of Images}
\author{Martin Brooks \\ \emph{martin.brooks@varilets.org} \thanks{\copyright 2016 Martin Brooks}}
\date{\today}

\begin{document}       
\maketitle 
\begin{abstract}
A \emph{persistence lens} is a hierarchy of disjoint upper and lower level sets of a continuous luminance image's Reeb graph, providing a contrast-invariant topological representation of image contrast variation.  
Pulled back to the image, the boundary components of a persistence lens's interior components are Jordan curves that serve as a hierarchical segmentation of the image, and may be rendered as vector graphics.
A persistence lens determines a \emph{varilet basis} \cite{varilets} for the luminance image, in which image simplification is a realized by subspace projection. 
Image scale space, and image fractal analysis, result from applying a scale measure to each basis function.

\end{abstract}

\pagebreak
\tableofcontents

\pagebreak
\section{Introduction}

Variational and morphological image processing \cite{Mumford_Shah, hierarchical_segmentation_4, scale_sets, hierarchical_segmentation_5} employ diverse types of image segmentation for subsequent piecewise approximation by smooth or constant functions. The varilet transform's \emph{lens} parameter \cite{varilets} plays this same role: a lens comprises a hierarchy of nested \emph{facets}, each bounded in the image plane by one or more Jordan level sets, thereby defining a multiresolution segmentation. 

A \emph{persistence lens} uses the the level sets of the critical points of persistence birth-death pairs \cite{EdelsbrunnerBook}.
Thus, regions of lesser contrast are contained in regions of greater contrast; the nesting structure is a contrast-invariant topological representation of contrast variation with the image.  


When augmented with a local geometric, topological or image measure, a persistence lens's segmentation hierarchy becomes a scale space, which for natural images may exhibit power law distributions over large regions. 

Varilet transforms apply to a continuous interpolation of the image's luminance channel, where the scalar pixel values are considered as a grid of samples on the image plane continuum. 
Subsequent analysis remains in the continuous domain, using vector graphics to realize image segmentation and simplification. 

\section{Example}
\label{example_section}

This section provides an annotated example, intended to create motivation and perspective for the technical sections that follow.

Figure \ref{original_image} is the original image from which the subsequent images are derived.
Figure \ref{luminance_image} is the luminance image, which defines the continuous interpolation acted upon by varilet transforms.
The figures represent the images in png format, at their native size for 72 dpi.

We will use a persistence lens generated by the algorithm described in section \ref{lens_algo_section}. 
The lens has 29,298 lens regions, for a total of 43,572 facets; the first level of the hierarchy has 6,308 facets; the hierarchy has maximum lens region  nesting depth 19.
 
Figure \ref{full_gray} visualizes the lens by its segmentation of the continuous image, with each each lens facet filled gray; figure \ref{full_color} shows each facet filled with color.
Figure \ref{full_close_1} is close-up of a portion of figure \ref{full_color}; figure \ref{full_close_2} is a further magnification.
Figure \ref{depth_1_color} shows all facets at the first level of the lens hierarchy, figure \ref{depth_1_contour} shows their boundary components, and figure \ref{depth_1_single_facet} highlights a single facet.

The images shown in figures \ref{full_gray} - \ref{depth_1_single_facet} are vector graphics; however, the figures represent the vector images in png format with sufficient resolution for modest enlargement.

Figure \ref{fractal} indicates image regions having fractal structure, by identification of a power law distribution of their facets' contrast. In section \ref{fractal_section} we discuss other measures that may also generate fractal structure, including facet area and topological total variation.

Figures \ref{scale_space} shows three images from an image scale space based on topological total variation.

\newgeometry{margin=1in}

\begin{figure}[H]\centering
\includegraphics[width=6.6805in]{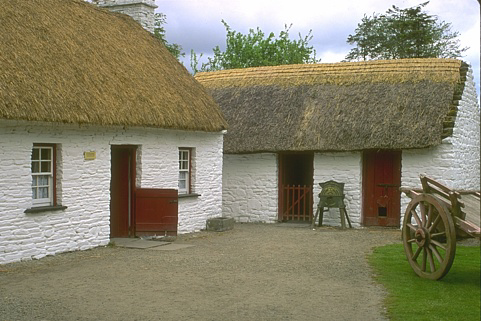}
\caption{Original image, a 481x321 8-bit RGB jpeg, number 385028 from the Berkeley Segmentation Database \cite{berkeley_database}.}
\label{original_image}
\end{figure}

\begin{figure}[H]\centering
\includegraphics[width=6.6805in]{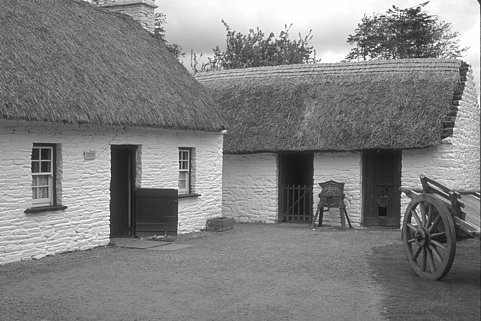}
\caption{Luminance image derived from the original. The varilet transform acts on scalar fields, thus the luminance image will be interpolated to a continuous function. The luminance values are treated as samples of a continuous function on the underlying plane continuum; the samples are located at the corners of a 480 x 320 grid of squares. Interpolation of each square patch is bilinear, except when bilinear interpolation creates a saddle point in the interior of the patch; in this case the square patch is replaced with
four triangular patches having a common vertex at the saddle point location, so that every patch is topologically monotone.}
\label{luminance_image}
\end{figure}

\begin{figure}[H]\centering
\includegraphics[width=6.6805in]{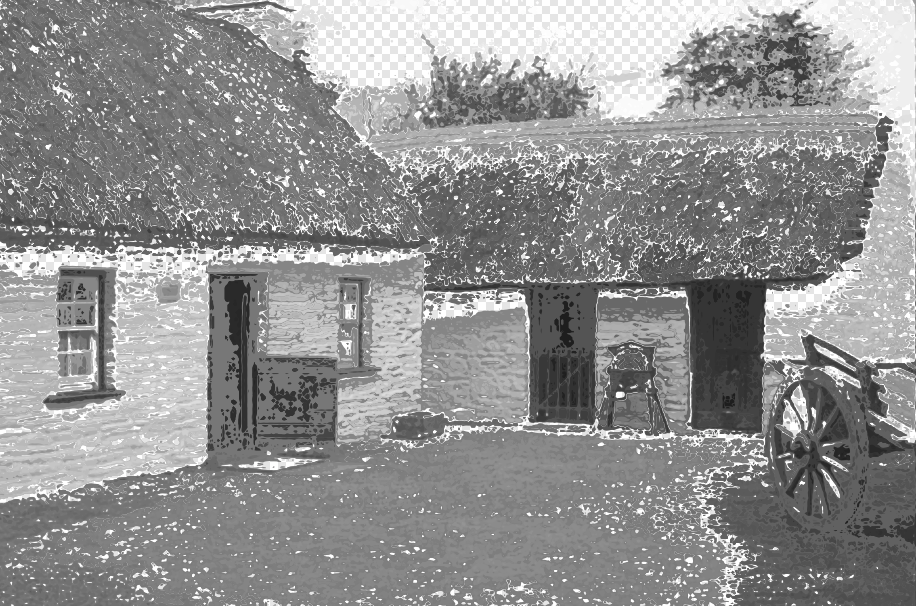}
\caption{Persistence lens visualized by its segmentation of the continuous image domain. 
Each of the 43,572 lens facets is filled with a single shade of gray, the luminance of the facet's boundary components. Each of the lens's facets is a connected open set; when it has more than one boundary component then the facet has holes.
Each boundary component is a Jordan (simple closed) curve upon which the  luminance is constant, with all components having this same value. 
As a vector image, each facet is represented as an SVG path element comprising of one or more closed polylines; they are correctly filled by the SVG even-odd fill rule.
}
\label{full_gray}
\end{figure}

\begin{figure}[H]\centering
\includegraphics[width=6.6805in]{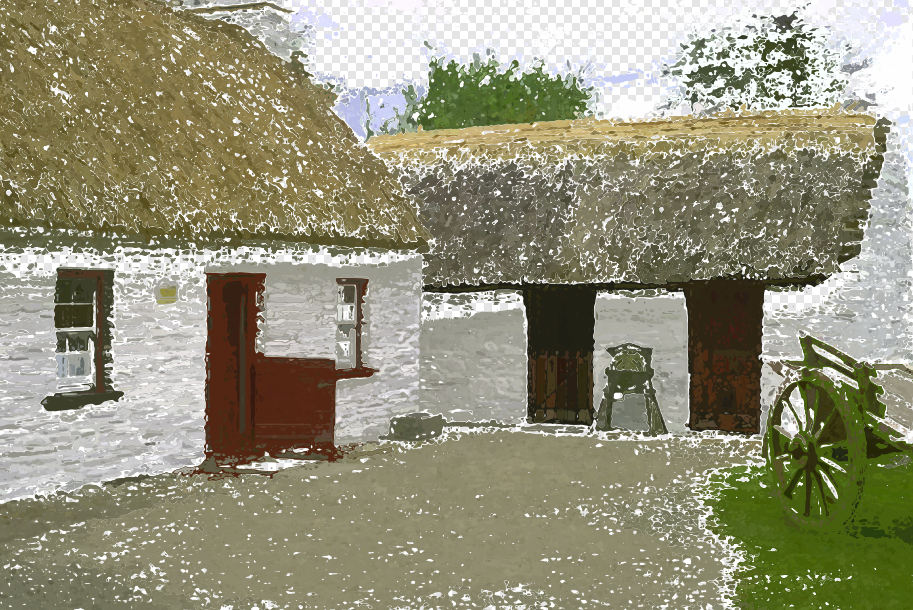}
\caption{Each lens facet filled with a single color, selected from the original image colors. As does figure \ref{full_gray}, this figure has unfilled regions, seen as the checkered background, which we now describe: The root of the lens hierarchy is the entire image; the unfilled regions of the image are the portion of the root not contained in any other facet. 
}
\label{full_color}
\end{figure}

\begin{figure}[H]\centering
\includegraphics[width=6.6805in]{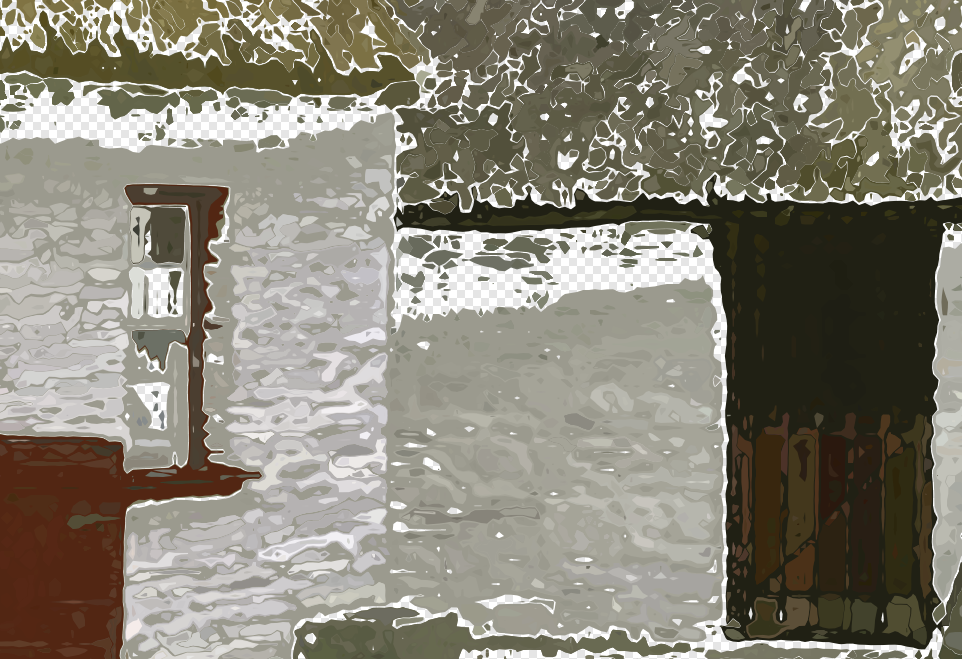}
\caption{Low magnification close-up of a portion of figure \ref{full_color}, showing lens facet nesting.}
\label{full_close_1}
\end{figure}

\begin{figure}[H]\centering
\includegraphics[width=6.6805in]{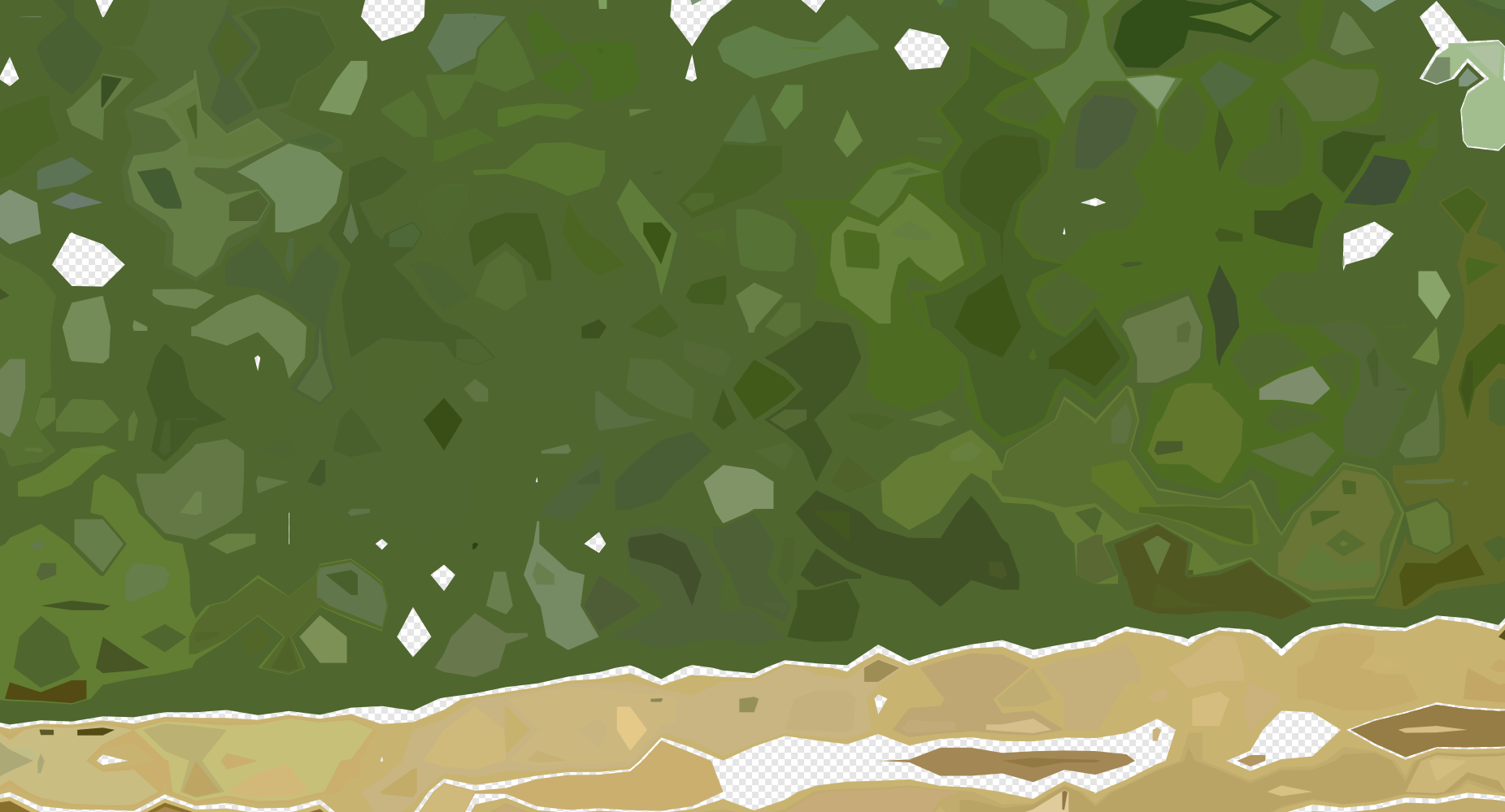}
\caption{High magnification close-up of a portion of figure \ref{full_close_1}, showing the level of detail arising from continuous interpolation of the image.}
\label{full_close_2}
\end{figure}

\begin{figure}[H]\centering
\includegraphics[width=6.6805in]{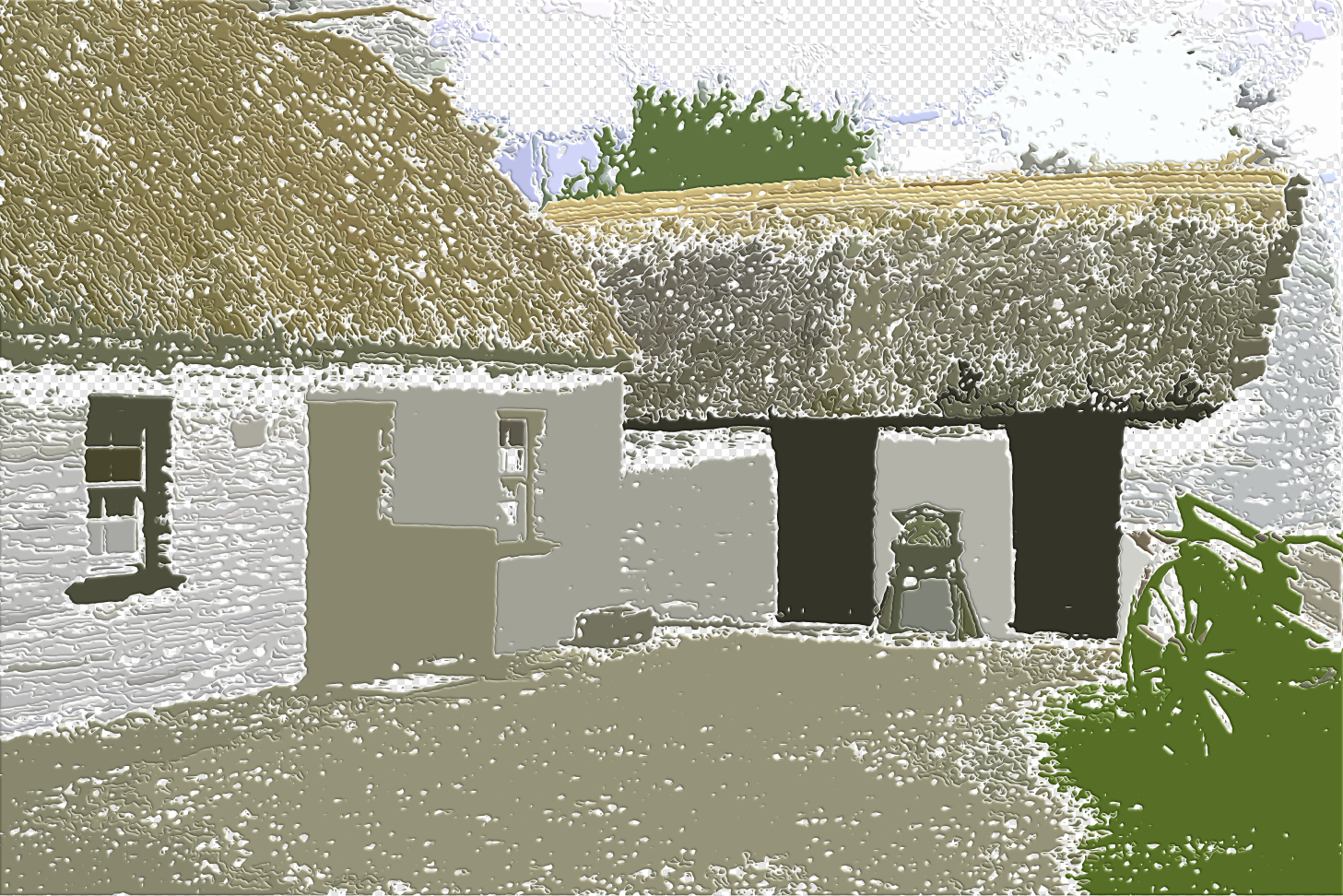}
\caption{Filled facets for the first level of the lens' hierarchy. These 6,308 facets are pairwise disjoint.}
\label{depth_1_color}
\end{figure}

\begin{figure}[H]\centering
\includegraphics[width=6.6805in]{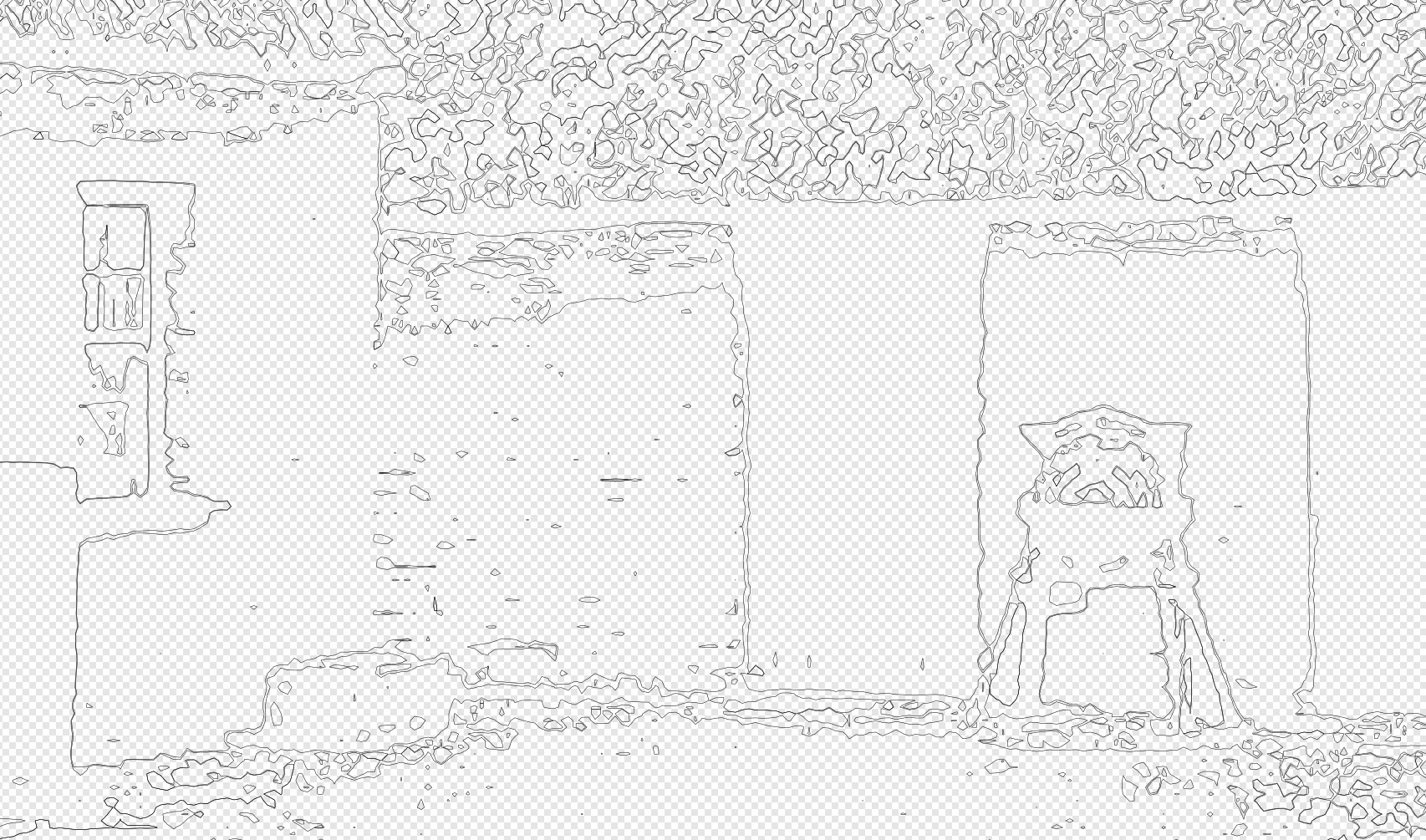}
\caption{Close-up of lens facet boundaries for the first level of the lens' hierarchy. We have not drawn the boundary components with the hyperbolic segments that result from bilinear interpolation; instead we have drawn straight line segments.}
\label{depth_1_contour}
\end{figure}

\begin{figure}[H]\centering
\includegraphics[width=4.0in]{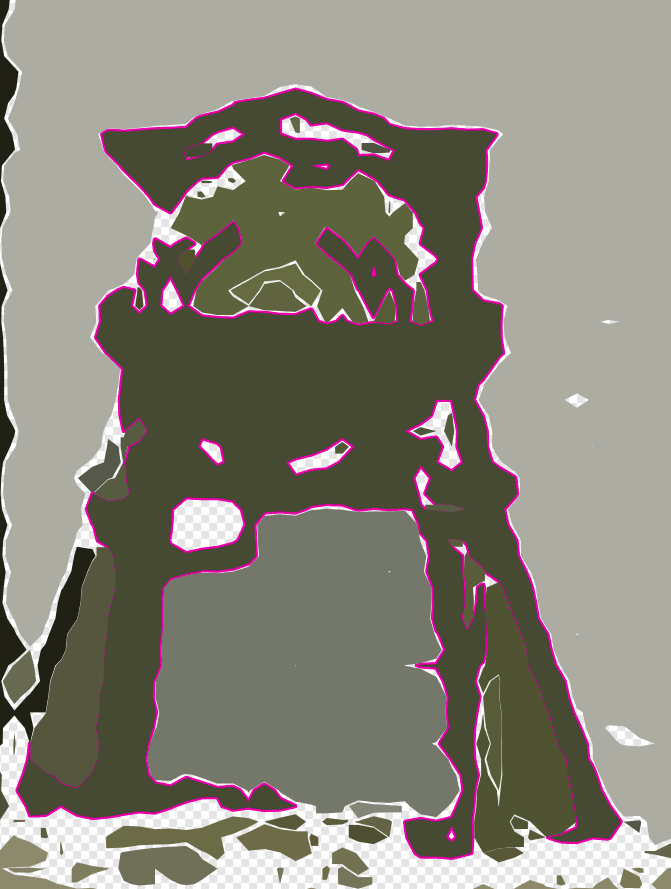}
\caption{Close-up of of a single facet at the first level of the lens' hierarchy, shown with pink outline.} 
\label{depth_1_single_facet}
\end{figure}

\begin{figure}[H]\centering
\includegraphics[width=6.6805in]{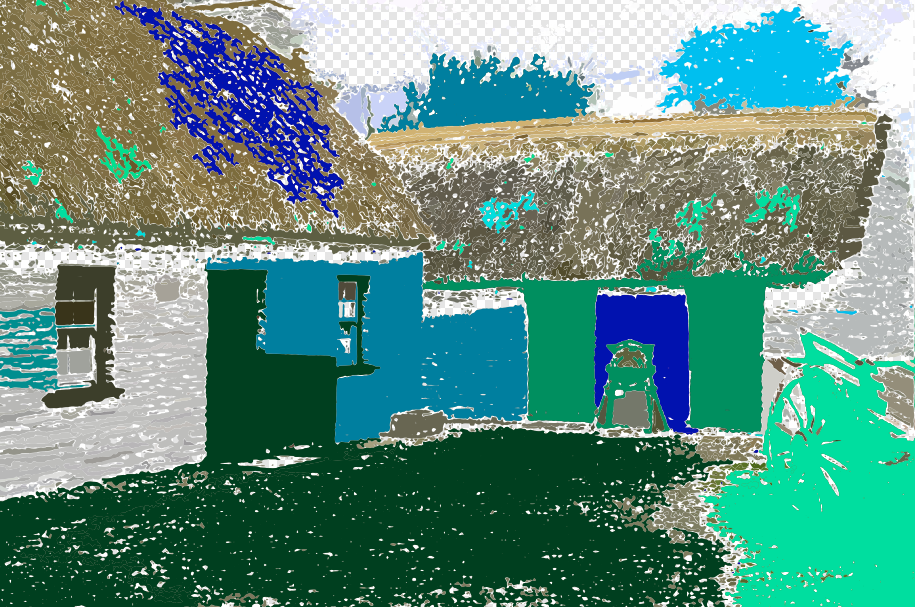}
\caption{First level facets of figure \ref{depth_1_color}, overlaid with brightly colored regions each having a distinct power law distribution of its sub-facets' contrasts. The power law exponents typically lie in the range 2-4, estimated by the maximum likelihood method of Clauset \cite{Clauset}.}
\label{fractal}
\end{figure}

\begin{figure}[H]\centering
\includegraphics[width=4in]{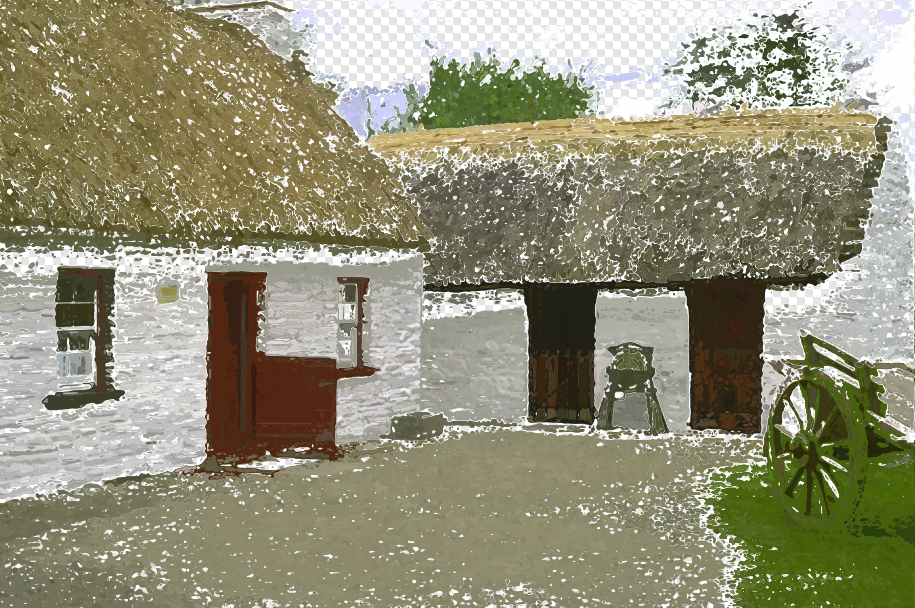}
\includegraphics[width=4in]{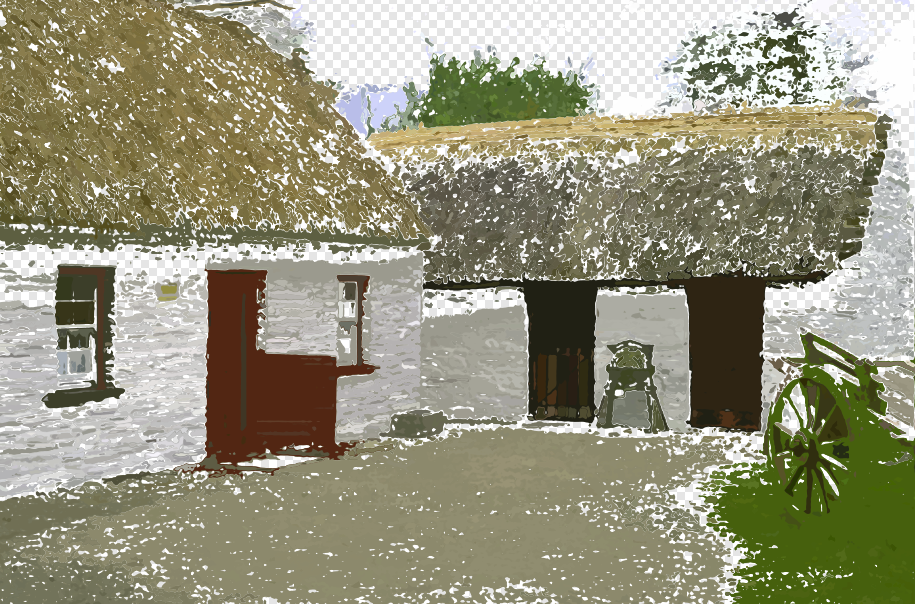}
\includegraphics[width=4in]{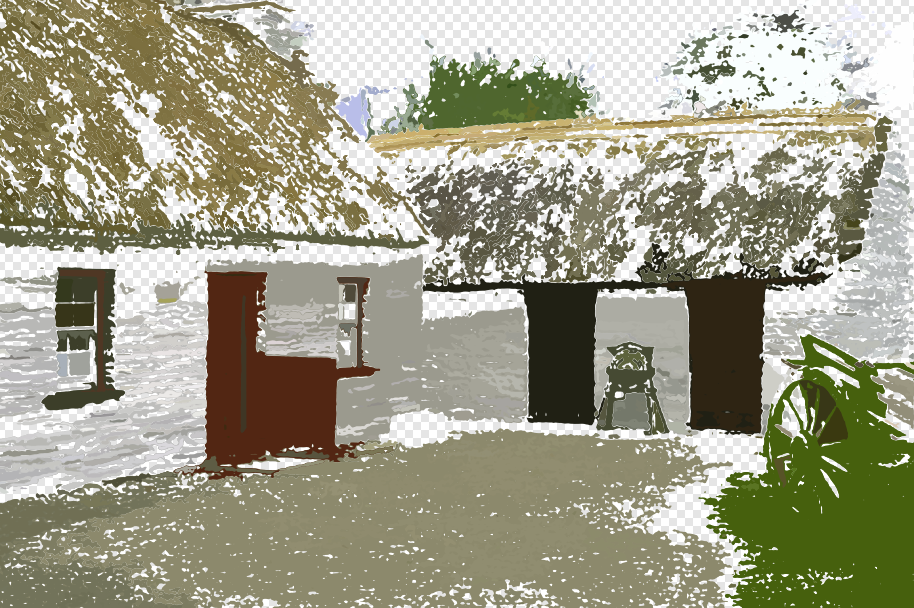}
\caption{Segmentation of lenses of successively coarser scale space images, using topological total variation as scale measure.}
\label{scale_space}
\end{figure}

\newgeometry{margin=2in}
\section{Theory}
\label{theory_section}

We provide an overview of varilet analysis for image processing.

Varilet analysis applies to real-valued continuous functions $f:X \to \Real$ on compact metric space $X$. 
For image analysis, $X$ is the image plane and $f$ is an interpolation of the images's luminance channel\footnote{Future research includes the possibility of using more advantageous ``scalarization'' of color images, such as Gooch et al.\ \cite{Gooch_gray}}.


All spaces in this paper are compact metric spaces and all functions are continuous.
For a subset $S$ of a topological space, we denote the interior by $S^{\circ}$, the closure by $\overline{S}$, the boundary by $\partial S$, and set difference by $S \smallsetminus T$.


\subsection{Varilets}

We briefly review the results of \cite{varilets}, starting with a classic result of analytic topology.

Continuous function $f:X \to Y$ is \emph{monotone} when $f^{-1}y$ is connected for all $y \in f(X)$; thus $f^{-1}$ carries connected sets to  connected sets.
$f$ is \emph{light} when $f^{-1}y$ is totally disconnected for all $y \in f(X)$.

The \emph{monotone-light factorization} \cite{Whyburn1942, walker1974, Lord1997} states that for every continuous function $f:X \to Y$ there exists a unique compact metric space $M$, called $f$'s \emph{middle space}, such that  $f =  \lambda \circ  \mu$, where $\mu :X \to M$ is monotone and $\lambda :M  \to Y$ is light. 
\begin{equation} 
\label{mlf_diag}
\includegraphics[width=2.0in]{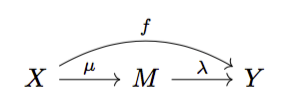}
\end{equation}

We restrict our attention to scalar fields, i.e. $Y = \Real$.
We denote $f$'s monotone-light factorization by $\mu M\lambda$.

\subsubsection{Piecewise Monotone Functions}


\begin{definition}
\label{pm_def}
Scalar field $f:X \to \Real$ is \emph{piecewise monotone} when  $M$ is a finite graph.
\end{definition}

For piecewise monotone $f$, the middle space $M$ is graph-theoretically and topologically identical to $f$'s Reeb graph \cite{Reeb1946}.
The monotone-light factorization is a useful context for the Reeb graph. 

For an image, where $f$ interpolates the luminance channel, $f$ can always be chosen to piecewise monotone\footnote{A forthcoming paper will discuss interpolation schemes for varilet analysis.}.



Most varilet analysis is couched in terms of $\lambda$ and $M$, subsequently using monotone factor inverse $\mu^{-1}$ to get back to  domain $X$. 
This is powerful, because $M$ enjoys the simple topology of graph continua \cite{Nadler1992}. 
Monotonicity of $\mu$ makes the theory oblivious to the many complexities of continua.

$f$'s \emph{critical points} are the vertices of $M$; points within an edge are \emph{regular points}.

The light factor $\lambda$ is numerically monotone along each edge of $M$. 
If vertex $p$ terminates both an increasing and a decreasing edge, then it is a \emph{saddle}; in this case $p$ terminates three or more edges. A vertex at which all edges have the same direction is an \emph{extremum}, either a maximum or minimum. $f$'s \emph{global extrema} are the points of $M$ mapped by $\lambda$ to $\lambda(M)$'s maximum and minimum values; all other extrema are \emph{local}. 

An interpolated image's domain is typically a rectangular region, for which the Reeb graph is a tree, also know as the \emph{contour tree} \cite{Carr2000}. However, data missing from the image may result in holes in the domain, or the image may have a topologically more complex domain, in which case the graph will have loops.
Varilet analysis applies equally well in all cases. 

\subsubsection{Varilets and Varilet Transforms}
\label{varilet_section}

We use the light factor $\lambda$ to measure length of edges in $M$: for edge $E$ terminated by vertices $p, q$, the length of $E$ is $|\lambda(p) - \lambda(q)|$.

\begin{definition}
\label{ttv_def}
\emph{Topological total variation} $\ttv(f)$  is the sum of all $M$'s edge lengths. 
\end{definition}

\begin{definition}
\label{varilet_basis_def}
A \emph{varilet basis} is a collection of piecewise monotone functions $\{g_i:X \to \Real\ |\  i = 0 \ldots N\}$, called \emph{varilets}, such that every linear combination $\sum a_ig_i$ has topological total variation
\begin{equation}
\label{ttv_equation}
\ttv(\sum a_ig_i) = \sum |a_i|\ttv(g_i)\ . 
\end{equation}
\end{definition}

Varilets are independent in the sense that $\sum a_ig_i = \sum b_ig_i$ only when all $a_i = b_i$.
Please note that equation (\ref{ttv_equation}) differs by a normalization factor from the formulation of \cite{varilets}.   

Topological total variation's relation to a varilet basis is analogous to that of energy for a finite Fourier basis.
The name \emph{\underline{vari}let} stems from this partitioning of topological total variation.

\begin{definition}
\label{varilet_transform_def}
A \emph{varilet transform} for $f$ is a varilet basis for which 
\begin{equation} 
\label{varilet_sum_eqn}
f =\sum g_i\ . 
\end{equation} 
\end{definition}

In \cite{varilets}, I provide a simple mathematical algorithm that  produces many different varilet transforms for any piecewise monotone $f$.
Each transform is specified by a special type of finite hierarchy on $f$'s middle space, which we call a \emph{lens}.

\subsubsection{Varilet Lens}

A varilet lens for scalar field $f$ is defined in terms of $f$'s monotone-light factorization $\mu M\lambda$.
A lens is a subset of the middle space $M$.

\begin{definition}
\label{varilet_lens}
\begin{tightList}
\item 
\item A subset $C \subset M$ is a \emph{lens region} for $f$ when $C$ is closed, connected, has nonempty interior, and $\lambda$ is constant on $\partial C$. 
\item A \emph{lens facet} is a connected component of a lens region's interior.
\item A \emph{varilet lens} for $f$ is a finite collection $\cC = \{C_i\ |\ i = 0 \ldots n\}$ of lens regions, with $C_0 = M$,  such that any two $C_i, C_j$ are either nested or disjoint.
\end{tightList}
\end{definition}

Let  $\cC = \{C_i\ |\ i = 0 \ldots n\}$ be a varilet lens for $f$.
Each lens region $C_i \in \cC$ and each facet $F \subset C_i$ may be pulled back to $f$'s domain by $\mu^{-1}$. 

Because $\mu^{-1}(F)$ is connected and open, the boundary components of $\mu^{-1}(F)$ are equal-luminance Jordan curves. 



The lens hierarchy also organizes a lens's facets: Let $F^i_1 \ldots F^i_m$ be $C_i$'s facets.
When another lens region $C_k \subset C_i$, then each of $C_k$'s facets is a subset of some $F^i_j$\ .

\subsubsection{Varilet Supports}

Given a varilet lens  $\cC = \{C_i\ |\ i = 0 \ldots n\}$ for $f$, the varilet transform produces one varilet basis function $g_i$ for each lens region $C_i \in \cC$.
Then each basis function
 $g_i$'s \emph{support} is the following  subset $D_i \subset C_i$:

\begin{equation}
\label{support_eqn}
D_i = \overline{C_i \smallsetminus \cup\  \{C_j\ |\ C_j \subsetneq C_i\}}\ .
\end{equation}

Each varilet $g_i$ is \emph{non-constant} only on its support's inverse image $\mu^{-1}D_i$.\footnote{``Support'' traditionally means ``non-zero'', but here means ``non-constant''. }

For $i \neq 0$, support $D_i$ always contains at least one point of $\partial C_i$.

Suppose $C_k \subset C_i$ is an \emph{immediate successor} in the lens hierarchy, i.e.\  there does not exist lens region $C_j \in \cC$ with $C_k \subsetneq C_j \subsetneq C_i$. Then $D_i$ contains at least one point of $\partial C_k$.

This discussion enables us to define the \emph{half-open support} $\halfsupport{D_i} \subset D_i$:

\begin{equation}
\label{half_support_eqn}
\halfsupport{D_i} = D_i \smallsetminus \partial C_i\ .
\end{equation}


The collection $\halfsupport{D_0} \ldots \halfsupport{D_n}$ partitions the middle space $M$.

\subsection{Image Segmentation}
\label{segmentation_section}

A varilet lens' supports $D_0 \ldots D_n$  cover $M$, intersecting only at their boundaries.
This makes the supports' inverse images $\mu^{-1}(D_0) \ldots \mu^{-1}(D_n)$ a natural candidate for segmentation.

However, for image processing we take a different approach.
The image segmentation for lens $\cC = \{C_i\ |\ i = 0 \ldots n\}$ is defined as the collection of inverse images of all of $\cC$'s  facets; we will also refer to these as facets. 

As discussed above, each lens facet is bounded by one or more constant-luminance Jordan curves, and the facets inhabit $\cC$'s lens region hierarchy.



The figures of the section \ref{example_section} show the unfilled regions resulting from segmentation incompleteness.

The unfilled regions include image domain points that do not lie in the first level of the segmentation.
These will always exist, because the first level lens regions are closed and disjoint.

Additionally, for any first level lens region $C \in \cC$, its boundary's inverse image $\mu^{-1}(\partial C)$ may \emph{properly} contain the Jordan boundary components of its facets; this difference will also be part of the unfilled region.
The unfilled points of $\mu^{-1}(\partial C)$ include various contour phenomena, including junctions, crack tips \cite{Mumford_Shah}, and regions of nonempty interior.



Varilet analysis articulates multiscale Jordan boundary components, at the cost of not articulating more complex phenomena.
This may offer practical advantages for image processing:

\begin{tightList}
\item[$\bullet$] By working exclusively with Jordan boundary components, we avoid the full complexity of image contour lines.
\item[$\bullet$] The image topology data required for the hierarchy of Jordan boundary components is acquired as part of the construction of the monotone-light factorization, leveraging a rich literature on image Reeb graph construction, e.g.\ \cite{Carr2000, Chao}. 
\item[$\bullet$] Truncation of recursion depth corresponds to image simplification (section \ref{simplification_section}).  
\item[$\bullet$] Vector graphics realization of filled segments is immediate, with lens facet drawing order reflecting the lens hierarchy from root down to leaves.
\begin{tightList}
\item[-] Vector graphics admits application of continuous mathematics.
\item[-] Vector graphics supports exploratory data analysis by zooming.
\end{tightList}
\end{tightList}

\subsection{Vectorization of Segmentation}
\label{vectorization_section}

Vectorization of Jordan boundaries is accomplished as follows: $f$'s middle space is constructed by sweeping a plane through $f$'s graph. For every interval of luminance values \emph{not} containing a sample or critical point, we save the sequences of pixel coordinates through which the plane passes. Subsequently, when we need to vectorize a lens facet boundary having luminance value $L$, we use $L$ to lookup the coordinate sequence, then calculating the interpolated sub-pixel path of the boundary as an SVG polyline. 

The main complication stems from boundaries that intersect the image frame; these are completed to a simple closed curve by continuing around the frame until meeting the other end. Boundaries are \emph{oriented}, thereby providing disambiguation at several algorithmic junctures. 

A lens facet is rendered as an SVG draw command comprising the collection of closed polylines of its boundary components.

Lens facets are correctly filled, including holes, by SVG's even-odd fill rule.
Fill color may be chosen as the gray level of the facet's boundary (figure \ref{full_gray}), or may be selected from the pixel color statistics of the facet region.

\subsection{Image Simplification}
\label{simplification_section}

Given  interpolated luminance image $f$ and lens $\cC= \{C_i\ |\ i = 0 \ldots n\}$, the varilet transform yields a varilet basis $\{g_i\ |\ i= 0 \ldots n\}$  such that $f = \sum g_i$.
From this basis we may construct filtered  $f' = \sum a_ig_i$, where each $a_i \in \Real$.

We restrict our attention to \emph{binary filters}, where either $a_i = 1$ or $a_i = 0$.
A binary filter is \emph{proper} when at least one, but not all, $a_i = 0$.

Suppose $f$ is a continuous luminance image having monotone-light factorization $\mu M \lambda$.
Let $\cC= \{C_i\ |\ i = 0 \ldots n\}$ be a lens for $f$, and let $\{g_i\ |\ i= 0 \ldots n\}$ be the varilet basis resulting from the varilet transform. 

\begin{definition}
\label{simplification_def}
An \emph{image simplification} of $f$ is a proper binary filter $f' = \sum a_ig_i$ such that $a_i = 0$ implies $a_j= 0$ for all $C_j \subset C_i$ .
\end{definition}

The filtered function $f'$ defines a gray vector image, using a continuous gray scale $[0\ 255]$. 

In the following subsections we discuss characteristics of simplified images. 

\subsubsection{Simplification Kernel and Cokernel}

Suppose lens $\cC= \{C_i\ |\ i = 0 \ldots n\}$ and binary coefficients $a_0 \ldots a_n$ define simplified image $f'$.
Letting $D_0 \ldots D_n$ denote the varilet supports, we define the simplification's \emph{kernel} $\kernel$ and \emph{cokernel} $\cokernel$ as the following subsets of $f$'s middle space $M$:

\begin{align*}
\kernel &= \cup\ \{D_i\ |\ a_i = 0\}\\
\cokernel &= \cup\ \{D_i\ |\ a_i =1\}\ .
\end{align*}

$\kernel$ and $\cokernel$ cover $f$'s middle space; they intersect only at their boundaries, which are equal. 

From definition \ref{simplification_def} it follows that the components of $\kernel$ are lens regions.

\subsubsection{Simplification as Quotient}


Simplification $f'$'s monotone-light factorization has middle space $M'$ which is a topological quotient of $M$.
The quotient map $\phi:M \to M'$ is monotone.

$\phi$ is a homeomorphism on the cokernel's interior $(\cokernel)^\circ$; whereas each kernel component is mapped to a distinct point.

$f'$'s monotone-light factorization $\mu'M'\lambda'$  is diagrammed in 
 (\ref{simplfication_diagram}), wherein function $\lambda^\flat:M \to \Real$ is identical to $\lambda$ except that $\lambda^\flat$ is constant on each kernel component, and $\phi^{r}$ denotes any right inverse of $\phi$.

\begin{equation}
\label{simplfication_diagram}
\includegraphics[width=2.0in]{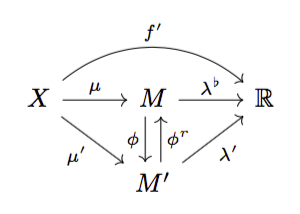}
\end{equation}

\subsubsection{Varilet Transform of Simplification}
\label{transform_simp_section}

From $f$'s lens $\cC= \{C_i\ |\ i = 0 \ldots n\}$ we may construct a \emph{simplified lens} $\cC'$ for simplification $f' = \sum a_ig_i$\ :

\begin{equation}
\label{simplified_lens}
\cC' = \{\phi(C_i)\ |\ a_i = 1\}
\end{equation}

Using this simplified lens, the varilet transform of $f'$  yields varilet basis $\{g_i\ |\ a_i = 1\}$, i.e.\  the basis of the simplification is the obvious subset of the original image's basis, here expressed in terms of varilet supports in the cokernel:

\begin{equation}
\label{simplified_sum_eqn}
f' = \underset{D_i \subset \cokernel}{\sum} g_i\ .
\end{equation}

Comparing equations (\ref{simplified_sum_eqn}) and (\ref{varilet_sum_eqn}), we see that the simplification is simply the original sum \emph{without} the terms for varilets having support in the kernel. This provides a useful characterization of the varilet transform: 
By transforming $f$'s representation to varilet basis functions $g_i$, simplification corresponds to deletion of terms from the sum $f = \sum g_i$\ ; i.e.\ $f'$ is the projection of $f$ to the subspace spanned by $\{g_i\ |\  a_i =1\}$\ .

When drawing varilet segmentations as vector graphics, the segmentation of a simplification is simply a truncation of the hierarchy of nested facets.

\subsection{Critical Point Trackability and Stability}
\label{trackability_section}

The early scale space papers of Witkin \cite{Witkin83} and Koenderink \cite{Koenderink} identify desiderata for scale space, paraphrased and contextualized as follows. 

\begin{tightList}
\item[(1)] Critical points can be tracked across scales.
\item[(2)] No new critical points are created as scale increases.
\item[(3)] Critical points have stable magnitude and location across scales.
\end{tightList}

By \emph{trackability} we mean \#1 \& 2; by \emph{stability} we mean \#3.



We discuss relationships between the critical points of $f$ and those of simplified image $f'$, using lens $\cC$, quotient map $\phi$, kernel $\kernel$ and cokernel $\cokernel$.

\begin{tightList}
\item{[A]}
Each critical point in the interior of kernel $\kernel$ is \emph{removed}. 
In other words, \emph{no} critical point $p \in \kernel^\circ$ persists through this simplification; these critical points track to \emph{nowhere}.
\item{[B]}
Each critical point $p$ in the interior of cokernel $\cokernel$ is \emph{retained}. 
In other words, \emph{every} critical point $p \in \cokernel^\circ$ persists, without change, through this simplification; these critical points track
one-to-one.
\item{[C]} 
For any critical point $p \in \partial \kernel$, whether, and in what form, $p$ persists through this simplification depends on information not found in the general case. 

\end{tightList}

Statements [A] and [B] make critical point tracking easy, whereas [C] requires an additional definition in order to ensure desideratum \#2.

\begin{definition}
\label{track_def}
Lens $\cC$ is \emph{trackable} when each lens region $C \in \cC$ contains a critical point in its boundary.
\end{definition}

Definition \ref{track_def} ensures desiderata \#1\&2.


Regarding desideratum \#3, when lens region $C \in \cC$ is a kernel component, then $f'$ has constant value $\lambda(\partial C)$ on $\mu^{-1}(C)$, and thus critical point values are stable. Critical point ``location'' can be considered stable in the limited sense that $p \in \phi^{-1}(p')$ implies $\mu^{-1}(p) \subset {\mu'}^{-1}(p')$, with equality holding in case [B].

\subsection{Persistence Lens}

The notion of lens is quite general, but for image processing we focus on \emph{persistence lenses}, utilizing the contours of critical points of persistence birth-death pairs \cite{EdelsbrunnerBook}. In this section we describe the construction and properties of persistence lenses.

The power of persistent homology \cite{EdelsbrunnerBook} for topological data analysis \cite{Carlsson2009} derives in part from its context in algebraic topology; however in this paper take a simpler approach for more limited results.


\subsubsection{Persistence}
\label{persistence_region_section}

Suppose scalar field $f$ has monotone-light factorization $\mu M\lambda$.
Working exclusively in $f$'s middle space $M$:

\begin{definition}
\label{persis_def}
Local maximum $p$'s \emph{persistence region} $\prgn{p}$ is the largest lens region containing $p$ such that:

\begin{tightList}
\item[] $\max\ \lambda(\prgn{p}) = \lambda(p)$\ and
\item[] $\min\ \lambda(\prgn{p}) = \lambda(\partial \prgn{p})$\ .
\end{tightList}

The definition is symmetric for minima.
For global extremum $p$, we define $\prgn{p}=M$.

Local extremum $p$'s \emph{persistence} is 
\begin{equation}
\per{p} = |\lambda(p) - \lambda(\partial \prgn{p})|\ .
\end{equation}
For global extrema, $\per{p}$ is equal to the length of $\lambda$'s image.
\end{definition}

Persistence regions of same-sense extrema are either disjoint or nested.
On the other hand,  persistence regions of opposite-sense extrema may have intersecting interiors without being nested. This situation is closely related to the problem of \emph{simplification conflicts} described by Edelsbrunner et al. \cite{Edelsbrunner2002}.

As in Agarwal et al.\ \cite{Agarwal2006} and Bauer et al.\ \cite{Bauer2014a}, persistence regions can be computed in two independent global sweeps of $f$'s middle space, with the up-sweep capturing minima and the down-sweep capturing maxima. 

We will use \emph{paths} in $M$. A path has no self-intersections. A path has a direction, starting at some point and ending at another.
Concatenation of paths $Q$ and $R$ is denoted $QR$. 

For local maximum $p$ (with a symmetric statement for minima), maximality of $\prgn{p}$ implies that there exists a critical point $q \in  \partial \prgn{p}$ and a path $R$ from $q$  to a point $r \notin \prgn{p}$ such that:

\begin{tightList}
\item[]  $\lambda(r) > \lambda(p)$\  and
\item[]  $R \cap \prgn{p} = \{q\}$\  and 
\item[] $\lambda(R \smallsetminus \{q\}) = \intervalOC{\lambda(q)}{\lambda(r)}$\  .
\end{tightList}

We call $q$ an \emph{apogee}, $r$ a \emph{dominator}, and $R$ a \emph{dominator path}.
The designation  ``apogee'' is a many-to-many generalization of persistence pairing;  ``dominator'' corresponds to the Elder Rule of Edelsbrunner \cite{EdelsbrunnerBook}.

The collection of all of $p$'s apogees is denoted $\ap{p}$. Not necessarily every point of $\partial \prgn{p}$ is an apogee, but there exists at least one, and (by definition) $\ap{p} \subset \partial \prgn{p}$.
The possibility of multiple apogees follows from the possibility of equal $\lambda$ values. One may easily construct examples for which two local extrema share an apogee. 
 Global extrema have no apogees.
 
 Any path $Q \subset \prgn{p}$ from $p$ to an apogee $q$ is an \emph{apogee path}. 
Concatenation of an apogee path $Q$ and an dominator path $R$ is a \emph{persistence path} $P = QR$, 
denoted $QR = \perboom{p}{q}{r}$, thereby indicating an extremum $p$, apogee $q$, and dominator $r$.

\subsubsection{Persistence Lens Definition}

Consider varilet lens  $\cC = \{C_i\ |\ i = 0 \ldots n\}$ having half-open varilet supports $\halfsupport{D_0} \ldots \halfsupport{D_n}$.

\begin{definition}
\label{pc_def}
Lens region $C_i$ is  \emph{persistence closed} when local extremum $p \in C_i^\circ$ implies $\ap{p} \subset C_i$. 
\end{definition}

\begin{definition}
\label{ed_def}
Lens region $C_i$ is  \emph{externally dominated} when local extremum $p \in \halfsupport{D_i}$ implies that $p$ has a dominator $r$
such that $r \not\in C_j^\circ$ for all $C_j \subsetneq C_i$.
\end{definition}

\begin{definition}
\label{persistence_lens_def}
Varilet lens  $\cC$ is a \emph{persistence lens} when $\cC$ is trackable, and every lens region $C \in \cC$ is persistence closed and
externally dominated.
\end{definition}


\subsubsection{Extremal Tracking and Persistence Semi-stability}

A persistence lens ensures the following \emph{extremal tracking} and \emph{persistence semi-stability} properties, proved in appendix \ref{extremal_tracking_section}.

\begin{prop}
\label{extremal_prop}
When using a persistence lens, then for each extremum $p'$ of any simplification $f'$, there exists a same sense extremum $p$ of $f$ such that $p \in \phi^{-1}(p')$, and for every such $p$ we have $\per{p} \ge \per{p'}$.
\end{prop}

In the proposition, $\phi$ is the monotone quotient map of diagram (\ref{simplfication_diagram}).

This is a ``sense making'' result, extending trackability to extrema and providing a limited form of stability of persistence: Every extremum of a simplified image is \emph{tracked to} by one or more extrema in the original image, and persistence does not grow. 
The particulars of definition \ref{persistence_lens_def} (persistence lens) were chosen specifically for this reason.

\subsubsection{Construction of Persistence Lens by Conflict Resolution}
\label{lens_algo_section}

Let $\cC^{\uparrow}$ be the collection of all persistence regions $\prgn{p}$ for minima $p$ of $f$; and similarly $\cC^{\downarrow}$ for maxima\footnote{The arrows reflect the sweep direction.}.
Each of $\cC^{\uparrow}$ and $\cC^{\downarrow}$ is a persistence lens.

However, $\cC^{\uparrow} \cup\ \cC^{\downarrow}$ is not in general a varilet lens.
We say that $C_1 \in \cC^{\uparrow}$ \emph{overlaps} $C_2 \in \cC^{\downarrow}$ when they have non-empty intersection but are not nested.

We construct a persistence lens by an iterative process of \emph{conflict resolution} for overlapping regions.
There are multiple ways to resolve conflicts; this is a source of multiplicity of persistence lenses.

We construct a persistence lens by first constructing $\cC^{\uparrow} \cup\ \cC^{\downarrow}$, and then removing or substituting for every pair of overlapping regions, iterating this process until no conflicts remain.

Suppose $C_1 \in \cC^{\uparrow}$ overlaps $C_2 \in \cC^{\downarrow}$.
One may resolve this conflict by simply omitting both regions from the final result.
However, we want lenses to have more structure, not less, and therefore are motivated to find alternatives.
Two solutions that we currently use for images include the following\footnote{Addition additional conflict resolutions are possible.}:

\begin{tightList}
\item[\emph{Choice}]: Choose one of $C_1$ and $C_2$. 
\item[\emph{Union}]: When $\lambda(\partial C_1) = \lambda(\partial C_2)$, substitute lens region $C_1 \cup C_2$.
\end{tightList}

Without getting into the details of the iteration, we can sketch an inductive proof that iterated conflict resolution results in a persistence lens. Starting with  $\cC^{\uparrow} \cup\ \cC^{\downarrow}$, all persistence lens requirements \emph{except nesting} are met. This situation obtains after applying the conflict resolution rules. The iteration halts when the lens regions are nested.

\subsection{Image Scale Space}
\label{image_scale_space_section}

In this section we endow varilet basis functions with scale measures.
The smallest detail appearing in $f$ when viewed through lens $\cC$ is expressed in terms of varilet transform $\{g_i\ |\ i= 0\ldots n\}$ as the minimum value of $\scale(g_i), i=0 \ldots n$\ .
Simplification can \emph{increase} scale by removing small detail.

\subsubsection{Scale Measures}
\label{scale_measure_section}

A \emph{scale measure} $\scale$ assigns to each piecewise monotone function $g:X \to \Real$ a nonnegative number $\scale(g)$. 
Scale measures will be applied to varilets.

We currently use the following scale measures; many more are possible.
This list comprises one geometric, one topological, and one image measure.

\begin{tightList}
\item The area of $g$'s support.
\item Topological total variation $\ttv(g)$\ .
\item Contrast $||g|| = \max\ g(X) - \min\ g(X)$.
\end{tightList}

\subsubsection{Scale Space Simplification}
\label{scale_space_section}

Thresholding scale measure $\scale$ at value $T$ provides the cokernel of a \emph{scale space simplification} by pruning the  hierarchy of persistence lens $\cC = \{C_i\ |\ i = 0 \ldots n\}$\ : 

\begin{equation}
\cokernel = \cup\ \{ D_i\ |\ C_j\supset C_i \text{ implies } \scale(D_j) \ge T\}\ .
\end{equation}

As expressed by equation (\ref{simplified_sum_eqn}), the resulting simplified image $f'$ has a naturally induced varilet lens for which the varilet basis is a subset of $f$'s varilet basis, each having $\scale(g_i) \ge T$.

By varying the threshold $T$ we may generate the collection of all \emph{distinct} scale threshold simplifications of $f$; this is $f$'s \emph{scale space} for persistence lens $\cC$ and scale measure $\scale$.

Not every simplification is a scale space simplification.
We measure the scale measure's \emph{fit} to lens $\cC$ as the fraction  of simplifications that are scale space simplifications; this is typically in the range $85 - 98\%$.

\subsection{Image Fractal Regions}
\label{fractal_section}

A \emph{fractal function}'s graph has non-integral Hausdorff dimension,  the function's \emph{fractal dimension} \cite{Mandelbrot}. 
Fractal analysis of multiscale data makes use of a variety of empirical \emph{fractal indices}, some of which correspond to fractal dimension, whereas others  stand on their own merit \cite{Mandelbrot, Theiler}. 

We use the following \emph{size counting} paradigm for fractal analysis:

\begin{tightList}
\item[(1)] Let $\cS$ be a multiset of nonnegative numbers, each of which is considered to be a measurement of the ``size'' of a ``feature'' of $f$.
\item[(2)]  Create the empirical distribution $\cD$ of \emph{feature size counts} from $\cS$.
\item[(3)]  Determine whether $\cD$ has a power law distribution, and if so, estimate the exponent.
\end{tightList}

For step 3, Clauset et al. \cite{Clauset, virkar2014} provide a maximum likelihood estimator for the power law exponent, using the Kolmogorov-Smirnov statistic for goodness of fit. 

We determine the fractal characteristics of image $f$ using a persistence lens $\cC = \{C_i\ |\ i = 0 \ldots n\}$ and scale measure $\scale$. 

Suppose lens  $\cC$ gives varilet basis $g_0 \ldots g_n$ having supports $D_0 \ldots D_n$.
To view a lens region $C_i \in \cC$ as a fractal, we use Clauset's maximum likelihood estimator on the distribution of varilet scales $\{\scale(g_j)\ |\ D_j \subset C_i\}$.

We calculate the power law exponent and goodness of fit for every lens region $C_i \in \cC$, and then aggregate the largest-possible regions having consistent exponent and high goodness of fit, resulting in analysis as shown in figure \ref{fractal} of section \ref{example_section}. 

Varying the choice of lens and/or scale measure causes minor variation of the fractal regions, an area of ongoing research.

\subsection{Theory Summary}
\label{theory_summary_section}

Varilet analysis is a novel and elementary extension of classic results of analytic topology \cite{Whyburn1942}.
Varilet analysis may have benefits as part of the larger image processing toolbox.
These include theoretically and algorithmically elementary forms of:

\begin{tightList}
\item[$\bullet$] multiresolution analysis,
\item[$\bullet$] vector graphics display,
\item[$\bullet$] scale space and fractal analysis.
\end{tightList}

Varilet analysis focuses on  \emph{monotonicity} in several guises, e.g.\ \emph{topological} monotonicity of $\mu$; \emph{numerical} monotonicity of $\lambda$ along the graph edges of $M$; monotone \emph{quotient map} $\phi$, and \emph{piecewise} monotonicity of $f$. In this respect varilet analysis may be complementary to existing methods.

By transforming $f$'s representation to varilet basis functions $g_i$, image simplification corresponds to deletion of terms from the sum $f = \sum g_i$.

Because varilet image analysis happens in the middle space, it relies on data collected during the monotone-light factorization in order to compute $\mu^{-1}$ and thereby pull simplifications back to the image space. The data for each varilet $g_i$ can be coded in various ways, and may be complete or incomplete. In the present application to image analysis, we chose to code Jordan boundaries only, leaving as unmodelled the more complex aspects of the image  (section \ref{segmentation_section})\footnote{It is possible to collect additional data for $\mu^{-1}$ in order to have more options and control of the rendered image representation, but we have not done so in this paper.}. 

Varilet analysis relies on \emph{finiteness} of the Reeb graph and \emph{continuity} of the luminance image.  This is a drastic mathematical simplification when compared to, say, functions of bounded variation. However, we see images as fundamentally finite, and we are content to allow closely spaced contours to straddle the ambiguity between continuous and discontinuous.

As an overall summary: Varilet image analysis provides many capabilities that are also provided by existing techniques, but that varilet analysis does so in a mathematically and algorithmically elementary way. This viewpoint is further explored in the next section.

\section{Related Work, Contributions and Discussion}
\label{related_work_section}

Varilet image analysis has commonalities with many theories. In this section we compare selected image processing approaches to varilets. 

\subsection{Reeb Graph \& Simplification}

Simplification of sampled two-dimensional scalar fields has appeared in work by Carr \cite{Carr2004}, Carr et al. \cite{CarrSP04}, Weber et al. \cite{Weber2007}, Bremer et al. \cite{Bremer2004}, Edelsbrunner et al. \cite{Edelsbrunner2002, Edelsbrunner2006}, Gyulassy et al. \cite{Gyulassy2006}, Bauer et al.\ \cite{Bauer2012}, and Tierny et al.\ \cite{Tierny2012, tiernyoptimal}.  

Each of these  references uses the topological structure of the scalar field to guide simplification. Carr et al. \cite{Carr2004, CarrSP04} and Weber et al. \cite{Weber2007} use the Reeb graph \cite{Reeb1946},  Bremer et al. \cite{Bremer2004} and Gyulassy et al.\ \cite{Gyulassy2006} use the Morse-Smale complex \cite{Edelsbrunner2003}, and Edelsbrunner et al. \cite{Edelsbrunner2002, Edelsbrunner2006} use the persistence diagram. The Reeb-based techniques are concerned with removing extrema; the Morse-Smale and persistence-diagram methods may also remove critical points related to the genus of isosurfaces. 

Varilet's contribution to Reeb graph analysis is formulation of simplification as a quotient (diagram \ref{simplfication_diagram}). This is accomplished by transforming $f$'s representation to a varilet basis $g_i$, in which simplification corresponds to deletion of terms from the sum $f = \sum g_i$.

In Carr et al. \cite{CarrSP04}, the order in which extrema are removed is determined by pruning contour tree leaves in preference order, using any of a variety of local geometric measures; this reference motivated varilets' use of geometric, topological and image measures (section \ref{image_scale_space_section}).

Piecewise monotone functions have identical middle space and Reeb graph. The monotone-light factorization entails additional information in the form of the monotone and light factors, as reflected in our notation $\mu M\lambda$.
There is nothing that prevents Reeb graph analysis from recognizing these functions. 
For example, light factor $\lambda$ (differently named) was used to show \emph{stability} of the Reeb graph under perturbations of $f$ by Bauer et al.\ \cite{Bauer2014a} and Di Fabio et al.\ \cite{DiFabio}.

Bauer et al.\ \cite{Bauer2014a} simplify by removing features having persistence below a threshold. Although varilets use persistence to \emph{define} lens structure, varilets do not use persistence to drive simplification,  due to a ``type'' mismatch:  We have assigned persistence to individual extrema, but we measure scale on a basis function. \emph{Contrast} is the luminance image measure (section  \ref{scale_measure_section}) most closely related to persistence; e.g.\ contrast \emph{is} persistence for extremal persistence regions (section \ref{persistence_region_section}).  

Bauer et al.\ \cite{Bauer2012} combine discrete Morse theory and persistent homology for function simplification guided by discrete vector fields. As do varilets, as well as \cite{Carr2004, Tierny2012}, they simplify by flattening.
Tierny et al.\ \cite{Tierny2012, tiernyoptimal} simplify scalar fields by working directly in the image space with guidance from the middle space topology, whereas varilets work directly in the middle space, lifting the result back to the image with $\mu^{-1}$ only at the end. 
Image-space simplification in \cite{Bauer2012,Tierny2012, tiernyoptimal} is powerful because it exercises full control over all details of the image. 
As discussed in sections \ref{segmentation_section}\ \&\ \ref{theory_summary_section}, this is in contrast to the varilet image processing, where we incompletely model image structure. 

Computational methods for simplification of three-dimensional visualization geometry use edge contraction in a triangular mesh \cite{Hoppe1996}. Some approaches include topological considerations based on the Reeb graph, e.g. Takahashi et al. \cite{Takahashi2004}. These works differ from varilet simplification, because they focus on simplifying the domain geometry of triangulated surfaces rather than simplifying scalar fields on a fixed domain. 

\subsection{Image Segmentation}

Mumford et al. \cite{Mumford_Shah} define a regularized equation whose solution provides a complete global image segmentation.
Their approach has been refined and solutions have been explored; for a review see \cite{Lemenant2016}.

Hierarchical image segmentation is also well-developed, including e.g.\ Abelaez et al.\ \cite{hierarchical_segmentation_1}, who provide the Berkeley Segmentation Data Set \cite{berkeley_database}.
Guigues et al.\ \cite{scale_sets, hierarchical_segmentation_2} use \emph{scale sets}, utilizing piecewise constant segmentation by regularizing within a hierarchy defined by persistence of regions. A similar approach is taken by Xu et al. \cite{hierarchical_segmentation_4}, with additional guidance from \emph{shape space} semantics.

Varilets' contribution to image segmentation is its mathematically elementary formulation of image segmentation as a hierarchy of open regions (lens facets) of the image space, each bounded by Jordan level sets. 

\subsection{Image Scale Space}

Starting with Witkin \cite{Witkin83} and Koenderink \cite{Koenderink}, \emph{scale space} has provided a parameterized family of smoothed images. For an overview, see Lindeberg \cite{Lindeberg}. 

Scale space theory includes both linear and nonlinear scale spaces. 
Linear scale spaces result from Gaussian smoothing, equivalently formulatied as heat diffusion \cite{Koenderink}.
Nonlinear scale spaces have various motivations, including the fact that certain types of multiresolution sensing systems are built around non-Gaussian filters \cite{Zuerndorfer}, a desire to extend the scale space concept to morphological filtering \cite{Chen1989, Morales}, and dissatisfaction with Gaussian filtering for vision applications \cite{Saint-Marc, Perona}. 

Varilet scale space \emph{looks} different than these references, but shares their basic intent: a sequential removal of detail, where at each stage the removed detail has smaller scale than what remains. Varilet scale space fully embraces the semantics of ``causality'' \cite{Koenderink}, using simplification's monotone quotient map to track identity (section \ref{trackability_section}). 

Reininghous et al.\ \cite{Reininghaus} define \emph{persistence scale space}, conceptually similar to varilets, but based on persistent homology and discrete Morse theory, whereas varilets are based on more elementary analytic topology.

Chen et al.\ \cite{Chen} study persistence diagrams norms as a function of the degree of scale space diffusion. Their experience of rapidly decreasing  norms is consistent with proposition \ref{extremal_prop}. Their figure 2 shows a linear log-log relationship between scale and number of extrema, indicating the possibility of a power law distribution; this would be similar for many natural images, due to naturally-occurring fractal content. We note that power laws constitute a relatively \emph{slow} decay; one may expect non-fractal images to exhibit exponential decay, further supporting the reference's observation.

Monasse et al.\ \cite{scale_space_level_lines} construct a scale space from level sets by representing them with all holes filled in.
By comparison, varilet analysis retains the holes of a level set's interior. 

Chao et al.\ \cite{Chao} construct a topologically multiscale scale space Reeb graph for content matching applications.

\subsection{Image Fractal Analysis}

Fractal structure of textures and natural images is well known; for example \ \cite{scaling_woods, turiel2000multifractal, fractal_texture}, and see the image processing applications at FracLab \cite{fraclab}. Blondeau et al.\ \cite{fractal_color} measure fractal color distribution. Various measurement and estimation schemes are applied, including wavelets, box-counting, energy and pixel methods. 

Varilet analysis is complementary to these methods, utilizing a different measurement and counting domain for power law estimation (section \ref{fractal_section}): For any choice of persistence lens, together with any choice of scale measure, the count of varilets by scale is input to Clauset's maximum likelihood estimator for power law exponent \cite{Clauset}  (figure \ref{fractal}). Again, this approach is simpler and more direct than many, and may therefore be useful.

\subsection{Jordan Boundaries}

It has been recognized, e.g.\ by Ambrosio et al. \cite{Ambrosio, Monasse}, that a level set's interior's boundary components are Jordan curves; these same references construct a hierarchy of Jordan curves boundaries with similarlies to varilet's. Varilet analysis differs from this work in two ways: (1) Whereas the references work in the image space, varilets work in the middle space; and (2) the references merge two distinct hierarchies of Jordan curves to get \emph{the} region hierarchy, whereas varilets utilize an externally supplied hierarchy in the form of a \emph{lens} parameter (section \ref{lens_algo_section}).  Different lenses may be used, in accordance with differing requirements and preferences.

\subsection{Image Vectorization}

Varilet image analysis provides vector representation of hierarchical image segmentation, by vectorizing the Jordan lens facets' boundary components. The luminance value is identical for all Jordan boundaries of the same lens region. 

Image vectorization methods typically use image segmentation and/or edge detectors, e.g.\ Selinger \cite{potrace}. Birdal et al. \cite{birdal} merge regions of similar color. Kopf et al.\ \cite{Kopf} use heuristics to group pixels into cells.

Orzan et al.\ \cite{Orzan}, Xie et al.\ \cite{Xie}, and Olsen et al.\ \cite{Olsen}  combine image simplification and vectorization using multiscale diffusion curves.

Fuchs et al.\ \cite{svg_streaming_1} provide a level-of-detail approach to progressive SVG imagery. Whereas their SVG is the entire image, varilets' SVG is the segmented image (section \ref{segmentation_section}); therefore the two methods are not addressing the same problems. However, varilets' lens hierarchy does provide a natural source for level-of-detail SVG streaming of the the segmented image.

\subsection{Image Total Variation}

Varilet analysis takes a new view by defining topological total variation as the sum of the Reeb graph's arc lengths (section \ref{varilet_section}).
$\ttv$ is fundamental; the varilet basis functions partition $\ttv$ in analogy to Fourier and wavelet partitions of energy.
$\ttv$ serves as a topological measure for scale space and fractal analysis.

Total variation image denoising is well known \cite{Rudin}.

Bauer et al.\ \cite{Bauer2010} link total variation and persistence in denoising.
Plonka et al.\ \cite{Plonka} apply a variant of total variation denoising employing a regularization term having persistence-derived coefficients.
Further research may show relationships to varilets' use of a persistence lenses for image processing.

\section{Conclusion}

We have presented a novel image processing approach having very direct expressions of image segmentation, simplification, vectorization, scale space and fractal analysis. 
The purpose of our exposition is to create awareness of varilet analysis as a basis for additional image processing tools. 

The author wishes to thank Dr. Michael Stieber at Apollo Systems Research Corporation, and also National Research Council Canada, for support of this research.

\appendix

\section{Appendix: Proof of Extremal Tracking and Persistence Stability}
\label{extremal_tracking_section}

Consider lens $\cC$ and simplification $f'$ having quotient map $\phi$, kernel $\kernel$, cokernel $\cokernel$, and monotone-light factorization $\mu'M'\lambda'$.

We prove proposition \ref{extremal_prop} in two parts in the following two subsections.

\subsection{Extremal Tracking}




\begin{prop}
\label{extremal_tracking_lemma}
Let $C_k$ be a persistence-closed kernel component containing no global extrema in its interior.
Then $\phi(C_k)$ is an extremum only when $\partial C_k$ contains a same-sense extremum.
\end{prop}

\begin{proof}
Suppose $\phi(C_k)$ is a maximum. Then for each boundary point $u \in \partial C_k$, every edge $E \ni u$ that does not lie entirely in $C_k$ terminates at a critical point $u^- \notin C_k$, with $\lambda(u^-) < \lambda(u)$. 

Assume $\partial C_k$ is comprised entirely of saddles and regular points; we will derive a contradiction. 
From this assumption we have:

\begin{tightList}
\item[\ \ \ ]Property X: For every $u \in \partial C_k$ there exists a maximum $p \in C_k^\circ$ and a path $P \subset C_k$ from $u$ to $p$ upon which $\lambda$ is only increasing. 
\end{tightList}

We now proceed with the proof. There exists at least one maximum $p \in C_k$ such that each of its dominators $r \notin C_k$;
for example, $p$ may be chosen by the condition $\lambda(p) = \max \lambda(C_k)$.

Let $p_1$ be such a maximum, chosen with the additional constraint that $\per{p_1}$ is minimal over all such maxima. 

Choose any dominator path $R$ for $p_1$, ending at a dominator $r_1$.

Let $u \in \partial C_k \cap R$ be the last boundary point along $R$; then along the subpath of $R$ that  starts at $u$ and continues outside $C_k$, let $t \notin C_k$ a point for which $\lambda(t)$ is least. 
Then $t \ne u$, and by definition \ref{persis_def} we know $\lambda(\ap{p_1}) < \lambda(t)$.

Let $p$ be the extremum stipulated by property $X$; we claim $\lambda(\ap{p}) > \lambda(t)$. We know $\lambda(p) \le \lambda(p_1)$, because otherwise $p$ would be a dominator for $p_1$; therefore $\lambda(r_1) >\lambda(p)$. Thus, if not $\lambda(\ap{p_1}) > \lambda(t)$, then $\lambda(\ap{p_1}) = \lambda(t)$ and $\ap{p} \not\subset C_k$; this contradiction proves the claim.

Now consider \emph{any} maximum $p \in C_k$ such that  $\lambda(\ap{p}) \ge \lambda(t)$, and such that there exists a path $P \subset C_k$ from $u$ to $p$ having $\min \lambda(P) > \lambda(\ap{p_1})$. Such maxima exist, because the maximum and path stipulated by Property X satisfy the criteria.
Then $\lambda(p) \le \lambda(p_1)$, since otherwise $p$ would be a dominator of $p_1$.
It follows that $\per{p} < \per{p_1}$. 

Now choose such a maximum $p_2$, with the additional constraint that $\lambda(p_2)$ is maximal over all such extrema. We claim that by maximality, $r \notin C_k$ for every dominator $r$.  The claim follows because the stipulated path $P$ from $u$ to $p_2$ can be concatenated with the persistence path $QR$ path from $p_2$ through some $q \in \ap{p}$ to same-sense extremum dominator $r$; this concatenated path $PQR$ satisfies $\min \lambda(PQR) > \lambda(\ap{p_1})$, and therefore does not lie in $C_k$.  

Finally, this last conclusion causes $\per{p_2} < \per{p_1}$ to contradict the minimality of $\per{p_1}$, thereby contradicting property X and the assumption that $\partial C_k$ does not contain a maximum. 
\end{proof}


\subsection{Non-Increasing Persistence}


%

\begin{prop}
\label{noninc_pers_lemma}
$\per{p'} \le \per{p}$ for every local extremum $p' \in M'$ and each extremum $p \in \phi^{-1}(p')$.
\end{prop}

\begin{proof}
Choose any local extremum $p' \in M'$, and then choose any same-sense extremum $p \in M$ such that $p \in \phi^{-1}(p')$.
Then $p \in \cokernel$ and $\lambda'(p') = \lambda(p)$.

Let $QR = \perboom{p}{q}{r}$ be a persistence path having dominator with $r \in \cokernel$.
Then $\lambda'(r') = \lambda(r)$, where $r' = \phi(r)$.

From $Q$ and $R$'s images $\phi(Q)$, $\phi(R)$ we define path $Q'R' \subset M'$, from $p'$ to $\phi(q)$, and then to $r'$. 
Since each kernel component is a constant-boundary region $C_k \in \cC$, there can be a loop within $\phi(Q)$ or $\phi(R)$ only if $Q$ or $R$ crosses multiple times in and out of $C_k$;
paths $Q'$ and $R'$ result from  excising all such loops.

We cannot derive a lower bound for $\per{p'}$, because we have no information about persistence paths for $p'$. 
But we can use $Q'R'$ to prove $\per{p'} \le \per{p}$ by showing that  $\per{p'} \ge \per{p}$ implies $\per{p'} = \per{p}$.
Assume $p$ to be a maximum. 

Neither $p$ nor $r$ lie in the interior of any kernel component, but points along path $QR$, including apogee $q$, may do so.
Consider any kernel component $C_k$ such that path $Q$ intersects $C_k^\circ$\ ; the following argument applies also when $R$ intersects $C_k^\circ$\ :
Because $\lambda'(\phi(C_k)) = \lambda(\partial C_k)$, we have
$\lambda'(\phi(C_k)) \ge \min \lambda(\phi^{-1}(Q \cap \phi(C_k)))$.
In other words, every point $s' \in Q'R'$ has 
\begin{equation}
\label{noninc_pers_eqn}
\lambda'(s') \ge \min \lambda(Q'R' \cap \phi^{-1}(s'))\ .
\end{equation}

Let $q'$ be the point along path $Q'R'$ at which $\lambda'$ takes its least value; if there are several points having this value, take $q'$ to be the last along the path. Let $Q''$ be the subpath of $Q'R'$ from $p'$ to $q'$, and let $R''$ be the subpath from $q'$ to $r'$.
Then by equation (\ref{noninc_pers_eqn}), $Q'' \subset \prgn{p'}$ and  $\lambda'(R'' \smallsetminus \{q'\}) = \intervalOC{\lambda'(q')}{\lambda'(r')}$\ .
Thus $\per{p'} \ge \per{p}$ only when $\per{p'} = \per{p}$, in which case $q' \in\ap{p'}$.
\end{proof}

\bibliography{varilets}
\bibliographystyle{plain}

\end{document}